\documentclass[11pt,times]{article}
\usepackage{amsmath,amsfonts,amssymb,amsthm}
\usepackage{times}
\usepackage{geometry}
\usepackage{algorithm}
\usepackage{algorithmic}
\usepackage{amsmath}
\usepackage{amssymb}
\usepackage{graphicx}
\usepackage{subfigure}

\usepackage{hyperref}

\makeatletter

\newtheorem{theorem}{Theorem}[section]
\newtheorem{lemma}[theorem]{Lemma}

\theoremstyle{definition}

\theoremstyle{definition}

\makeatother

\begin{document}

\title{Online Learning for Time Series Prediction}
\author{Oren Anava\\
\footnotesize{Technion, Haifa, Israel}\\ 
\footnotesize{soanava@tx.technion.ac.il}\\  \and Elad Hazan\\
\footnotesize{Technion, Haifa, Israel}\\ 
\footnotesize{ehazan@ie.technion.ac.il}  \and Shie Mannor\\
\footnotesize{Technion, Haifa, Israel}\\ 
\footnotesize{shie@ee.technion.ac.il}\\  \and Ohad Shamir\\ 
\footnotesize{Microsoft Research, MA, USA}\\
\footnotesize{ohadsh@microsoft.com}\\ }
\date{}
\maketitle

\begin{abstract}
In this paper we address the problem of predicting a time series using the ARMA (autoregressive moving average) model, under minimal assumptions on the noise terms. Using regret minimization techniques, we develop effective online learning algorithms for the prediction problem, \emph{without} assuming that the noise terms are Gaussian, identically distributed or even independent. Furthermore, we show that our algorithm's performances asymptotically approaches the performance of the best ARMA model in hindsight.
\end{abstract}

\section{Introduction}

A time series is a sequence of real-valued signals that are measured at successive time intervals. Autoregressive (AR), moving average (MA), and autoregressive moving average (ARMA) models are often used for the purpose of time-series modeling, analysis and prediction. These models have been successfully used in a wide range of applications such as speech analysis, noise cancelation, and stock market analysis (\cite{Hamilton94,BoxJenRe94,ShuSto05,BroDav09}). Roughly speaking, they are based on the assumption that each new signal is a noisy linear combination of the last few signals and independent noise terms.

A great deal of work has been done on parameter identification and signal prediction using these models, mainly in the ``proper learning" setting, in which the fitted model tries to mimic the assumed underlying model. Most of this work relied on strong assumptions regarding the noise terms, such as independence and identical Gaussian distribution. These assumptions are quite strict in general and the following statement from \cite{Thomson94} is sometimes quoted:
\begin{quote}
\textit{Experience with real-world data, however, soon convinces one that both stationarity and Gaussianity are fairy tales invented for the amusement of undergraduates.}
\end{quote}

In this paper we argue that these assumptions can be relaxed into less strict assumptions on the noise terms. Moreover, we offer a novel approach for time series analysis and prediction --- an \emph{online learning approach} that allows the noise to be arbitrarily or even (to some extent) adversarially generated.
The goal of this paper is to show that the new approach is more general, and is capable of coping with a wider range of time series and loss functions (rather than only the squared loss).\\

\subsection{Summary of results}
%\noindent \textbf{Summary of results}:
We present and analyze two online algorithms for the prediction problem, one designed for general convex loss functions and the other for exp-concave ones. Each of these algorithms attains sublinear regret bound against the best ARMA prediction in hindsight, under weak assumptions on the noise terms. We apply our results to the most commonly used loss function in time series analysis, the squared loss, and achieve a regret bound of $O \left( \log ^2 (T) \right)$ against the best ARMA prediction in hindsight. Finally, we present an empirical study that verifies our theoretical results.

\subsection{Related work}
%\noindent \textbf{Related work}:
In standard time series analysis, the squared loss is usually considered and the noise terms are assumed to be independent with bounded variance and zero-mean. In this specific setting, one can assume without loss of generality that the noise terms have identical Gaussian distribution (see \cite{Hamilton94,BoxJenRe94,BroDav09} for more information). This allows the use of statistical methods, such as least squares and maximum likelihood based methods, for the tasks of analysis and prediction. However, when different loss functions are considered these assumptions do not hold in general, and the aforementioned methods are not applicable. We are not aware of a previous approach that tries to relax these assumptions for general convex loss functions.
We note that there has been previous work which tries to relax such assumptions for the squared loss, usually under additional modelling assumptions such as \emph{t}-distribution of the noise (e.g., \cite{DaElS89,TiWoVaBi00}). We emphasize that the independence assumption is rather strict and previous works that relax this assumption usually offer specific  dependency model, e.g., as proposed by \cite{Engle1982} for the ARCH model.

Furthermore, an online approach that relies on regret minimization techniques was never considered for ARMA prediction, and hence regret bounds of the type we are interested simply do not exist. Yet, results on the convergence rate of the coefficient vectors do exist, and regret bounds can be derived from these results. E.g., in \cite{DinShiChe06} such results are presented, and a regret bound of   $O \left( \log ^2 (T) \right)$ can be derived for the squared loss. We are not familiar of these kind of results for general convex loss functions.

\section{Preliminaries and model}\label{sec:preliminaries}
\subsection{Time series modelling}
A \emph{time series} is a sequence of signals, measured at successive times, which are assumed to be spaced at uniform intervals. We denote by $X_t$ the signal measured at time $t$, and by $\epsilon_t$ the noise term at time $t$. The $\mathrm{AR}(k)$ (short for autoregressive) model, parameterized by a horizon $k$ and a coefficient vector $\alpha \in \mathbb{R}^k$, assumes that the time series is generated according to the following model, where $\epsilon_t$ is a zero-mean random noise term:
\begin{equation}
X_t = \sum_{i=1}^k \alpha_i X_{t-i}+ \epsilon_t.
\end{equation}
In words, the model assumes that each $X_t$ is a noisy linear combination of the previous $k$ signals. A more sophisticated model is the $\mathrm{ARMA}(k,q)$ (short for autoregressive moving average) model, which is parameterized by two horizon terms $k,q$ and coefficient vectors $\alpha\in \mathbb{R}^k$ and $\beta\in \mathbb{R}^q$. This model assumes that $ X_t $ is generated via the formula:
\begin{equation} \label{eq:arma}
X_t = \sum_{i=1}^k \alpha_i X_{t-i} + \sum_{i=1}^q \beta_i \epsilon_{t-i} + \epsilon_t ,
\end{equation}
where again $\epsilon_t$ are zero-mean noise terms. Sometimes, an additional constant bias term is added to the equation (to indicate constant drift), but we will ignore this for simplicity. Note that the $\mathrm{AR}(k)$ model is a special case of the $\mathrm{ARMA}(k,q)$ model, where the $\beta_i$ coefficients are all zero.

\subsection{The online setting for ARMA prediction}

Online learning is usually defined in a game-theoretic framework, where the data, rather than being chosen stochastically, are  chosen arbitrarily, possibly by an all-powerful adversary with full knowledge of our learning algorithm (see for instance \cite{CesaBianchiLu06}). In our context, we will describe the setting as follows: First, some coefficient vectors $(\alpha,\beta)$ are fixed by the adversary. At each time point $t$, the adversary chooses $\epsilon_t$ and generates the resulting signal $X_t$ using the formula in Equation \ref{eq:arma}.
We emphasize that $(\alpha,\beta)$ and the noise terms are not revealed to us at any time point.

At iteration $t$, we need to make a prediction $\tilde{X}_t$ for the signal, after which the real signal $X_t$ is revealed, and we suffer a \emph{loss} denoted by $\ell_t \big ( X_t , \tilde{X}_t \big)$.
Our goal is to minimize the sum of losses over a predefined number of iterations $T$. A reasonable benchmark is to try to be not much worse than the best possible ARMA model. More precisely, we let
\begin{equation} \label{ft}
f_t(\alpha, \beta) = \ell_t \big ( X_t ,  \tilde{X}_t (\alpha, \beta) \big) = \ell_t \left( X_t , \left( \sum_{i=1}^k \alpha_i X_{t-i} + \sum_{i=1}^q \beta_i \epsilon_{t-i} \right) \right)
\end{equation}
denote the loss at time $t$ of the (conditionally expected) prediction given by an ARMA model with some coefficients $(\alpha,\beta)$. We then define the \emph{regret} as
\begin{equation} \label{regret}
R_T = \sum_{t=1}^T \ell_t \big ( X_t , \tilde{X}_t \big) - \min_{\alpha,\beta} \sum_{t=1}^T \ell_t \big ( X_t ,  \tilde{X}_t (\alpha, \beta) \big).
\end{equation}
We wish to obtain efficient algorithms, whose regret grows sublinearly in $T$, corresponding to an average per-round regret going to zero as $T$ increases. \footnote{The iterations in which $t \leq k$ are usually ignored since we assume that the loss per iteration is bounded by a constant, this adds at most a constant to the final regret bound.}

%\subsection{Main challenge and technique}

A major challenge in our setting is that the noise terms $\{ \epsilon_t \}_{t=1}^T$ are {\em unknown}. As a result, we cannot use existing online convex optimization algorithms over the space of coefficient vectors $(\alpha,\beta)$. Moreover, even if we are given some $(\alpha,\beta)$, we cannot generate a prediction $\tilde{X}_t$  using the ARMA model. This lack of information makes it also hard to compute the best coefficient vectors in hindsight, and hence competing against the best ARMA model is ill-defined in this case.

%We overcome these difficulties via several steps. First, we reduce ARMA to an AR model of infinite horizon. We then show (this is the crucial step formally given in Lemma 3) that a finite memory, proportional to a logarithm in the total number of prediction iterations, is sufficient for prediction within optimal regret bounds. Finally, we use existing online learning techniques to learn the finite-memory AR model. 

\subsection{Our assumptions} \label{sec:assumptions}
Throughout Section \ref{sec:online prediction} we assume the following:
\begin{enumerate}
\item[1.]
The noise terms are stochastically and independently generated, each from a zero-mean distribution which might be chosen adversarially (up to the assumptions below). In Section \ref{sec:add} we show how to relax this assumption to adversarial noise.
Also, we assume that $ \mathbb{E} \left[ |\epsilon_t| \right] < M_{\max}<\infty$  and $ \mathbb{E} \left[ \ell_t \left( X_t , X_t - \epsilon_t \right) \right] < \infty $ for all $t$.
\item[2.]
The loss function $\ell_t$ is Lipshitz continuous for some Lipshitz constant $ L>0 $. This is a standard assumption and it holds in particular for the squared loss, as well as for other convex loss functions, with compact domain.
%The process $\{X_t\}_{t=1}^T$ is stationary. 
\item[3.]
The coefficients $\alpha_i$ satisfy $|\alpha_i|<c$ for some $c \in \mathbb{R}$. This assumption is also standard, and needed in general for the decision set (defined in Subsection \ref{sec:params}) to be bounded. We assume that $ c=1 $ without loss of generality.
\item[4.]
The coefficients $\beta_i$ satisfy $\sum_{i=1}^q |\beta_i| < 1-\varepsilon $, for some $ \varepsilon  > 0 $. 
\item[5.]
The signal is bounded (by constant which is independent of $T$). Without loss of generality we assume that $ \left| X_t \right| < 1 $ for all $t$.
\end{enumerate}

\section{Online time series prediction} \label{sec:online prediction}

As said before, we cannot use existing online convex optimization algorithms over the space of coefficient vectors $(\alpha,\beta)$ since the noise terms are unknown to us at any stage. Instead, we use an \emph{improper learning} approach, where our predictions at each time point will not come from an ARMA model that tries to mimic the underlying model. More specifically, we fix some $m\in\mathbb{N}$, and at each time point $t$, we choose an $\left( m + k \right) $-dimensional coefficient vector $\gamma \in \mathbb{R}^{m+k}$ and predict by $ \tilde{X}_t(\gamma) = \sum_{i=1}^{m+k}\gamma_i X_{t-i}$. It follows that our loss at iteration $t$ is determined by the loss function
\begin{equation} \label{lmt}
\ell^m_{t} (\gamma^t) = \ell_t \big( X_t , \tilde{X}_t(\gamma^{t}) \big)  = \ell_t \left( X_t , \left( \sum_{i=1}^{m+k} \gamma^{t}_i X_{t-i} \right) \right) .
\end{equation}
This can be seen as an AR model with horizon $\left( m + k \right)$. This leads to one of our key results: we can learn $\mathrm{ARMA}(k,q)$ model using $\mathrm{AR}(m+k)$ model, for a properly chosen value of $m$. We quantify this result in Theorem \ref{main} in terms of regret.

\subsection{Algorithm parameters definition and calculation} \label{sec:params}
Before presenting the algorithm and stating our main theorem, we need to define the following parameters. The decision set $\mathcal{K}$ is the set of candidates ($\left(m+k\right)$-dimensional coefficient vectors) we can choose from at each iteration; it is defined as
$$ \mathcal{K} = \left\{\gamma \in \mathbb{R}^{m+k} \ , \  |\gamma_j| \leq 1 \ , \ j = 1,\ldots,m \right\}. $$
Intuitively, the structure of $\mathcal{K}$ follows from Assumptions 3-4 on $\alpha$ and $\beta$, which restrict our improper learning variable $\gamma$. We denote by $D$ the diameter of $\mathcal{K}$, and bound:
\begin{equation} \label{eq:params_D}
D = \sup_{\gamma_1,\gamma_2 \in  \mathcal{K}} \| \gamma_1- \gamma_2 \| _2 = \sqrt{2 \left( m+k \right) }.
\end{equation}
Next, we denote by $G$ the upper-bound of $\| \nabla \ell^m_t(\gamma) \|$ for all $t$ and $\gamma \in \mathcal{K}$. This parameter depends on the loss function considered, and its computation is done accordingly. E.g., for the squared loss we get that $G= 2 \sqrt{m+k} D $, relying on Assumption 5. Finally, we denote by $\lambda$ the exp-concavity parameter of the loss functions $\{\ell_t^m\}_{t=1}^T$, i.e., it holds  that $ e^{-\lambda \cdot \ell^m_t(\gamma)} $ is concave for all $t$ \footnote{It is easy to show that every exp-concave function is convex, the opposite is not correct.}. This parameter is relevant only for exp-concave loss functions, and its computation is also done according to the loss function considered. It can be shown that $\lambda= \frac{1}{m+k}$ when the squared loss is considered.

\subsection{ARMA Online Newton Step (ARMA-ONS)} \label{sec:ONS}
Algorithm \ref{alg1}  shows how to choose $ \gamma^t $ in each iteration, when the loss functions $\{\ell_t^m\}_{t=1}^T$ are assumed to be $\lambda$-exp-concave in $\gamma$. The notation $\Pi_\mathcal{K}^{A_{t}}$ refers to the projection onto $\mathcal{K}$ in the norm induced by $A_{t}$, i.e., $
\Pi_\mathcal{K}^{A_{t}} (y) = \arg \min_{x \in \mathcal{K}} (y-x)^\top A_{t} (y-x).
$

\begin{algorithm}[h!]
\caption{ARMA-ONS(k,q)}
\label{alg1}
\begin{algorithmic}[1]
\STATE Input: ARMA order k,q; learning rate $\eta$;  an initial $\left( m + k \right)  \times \left( m + k \right)$ matrix $A_0$.
\STATE Set $ m = q \cdot \log_{1-\varepsilon } \left( \left(  T L M_{\max}  \right)^{-1}\right)  $.
\STATE Choose $\gamma^1 \in \mathcal{K}$ arbitrarily.
\FOR {$t=1$ to $(T-1)$}
\STATE Predict $ \tilde{X}_t(\gamma^{t}) = \sum_{i=1}^{m+k} \gamma^{t}_i X_{t-i}$.
\STATE Observe $X_t$ and suffer loss $\ell^m_t(\gamma^t)$.
\STATE Let $\nabla_t = \nabla \ell^m_t(\gamma^t)$, update $A_t \leftarrow A_{t-1} + \nabla_t \nabla_t^\top $
\STATE Set  $\gamma^{t+1} \leftarrow \Pi_\mathcal{K}^{A_{t}} \Big( \gamma^{t}-\frac{1}{\eta} A^{-1}_{t} \nabla_t \Big)$
\ENDFOR
\end{algorithmic}
\end{algorithm}

\noindent  In case the dimension $\left( m+k \right) $ of $A_t$ is large, we note that its inverse can be efficiently re-computed after each update using the Sherman-Morrison formula.\\

\noindent For Algorithm \ref{alg1} we can prove the following:
\begin{theorem} \label{main}
Let $k,q \geq 1$, and set $A_0 = \epsilon I_{m+k}$, $\epsilon = \frac{1}{\eta^2 D^2}$, $ \eta = \frac{1}{2} \min \{4GD, \lambda\}$. Then, for any data sequence $ \{X_t\}_{t=1}^{T}$ that satisfies the assumptions from Section \ref{sec:assumptions}, Algorithm \ref{alg1} generates an online sequence $\{\gamma^t\}_{t=1}^T$, for which the following holds:
\begin{eqnarray}
 \sum_{t=1}^T \ell^m_{t}(\gamma^{t}) - \min_{\alpha,\beta} \sum_{t=1}^T \mathbb{E} \left[ f_t(\alpha, \beta) \right]  = O \left( \left( GD+ \frac{1}{\lambda} \right) \log(T) \right).
\end{eqnarray}
\end{theorem}
\textbf{Remark}: The expectation is necessary since the noise terms $\epsilon_t$ are unknown. Also, obtaining a high probability bound on the regret is possible but requires additional assumptions on the noise process such as boundedness or light tail.\\

\begin{proof}
Intuitively, Theorem \ref{main} states that we can have a regret as low as the best $\mathrm{ARMA}(k,q)$ model, using only an $\mathrm{AR}(m+k)$ model. The proof consists of two steps. In the first step we bound the regret suffered by an $\mathrm{AR}(m+k)$ prediction using familiar techniques of online convex optimization. In the second step we bound the distance between the $\mathrm{AR}(m+k)$ loss function and the $\mathrm{ARMA}(k,q)$ loss function, using a chain of bounds and inequalities. Integrating both steps yields the requested regret bound for the $\mathrm{ARMA}(k,q)$ loss function.\\
\textbf{Step 1}: Relying on the fact that the loss functions $\{\ell_t^m\}_{t=1}^T$ are $\lambda$-exp-concave, we can guarantee that
\[
\sum_{t=1}^T \ell^m_{t}(\gamma^{t}) - \min_{\gamma} \sum_{t=1}^T \ell^m_{t}(\gamma) = O \Big(\Big(GD+\frac{1}{\lambda}\Big) \log(T) \Big) ,
\]
using the \emph{Online Newton Step} (ONS) algorithm, presented in \cite{DBLP:journals/ml/HazanAK07}.\\
\textbf{Step 2}: Define recursively
\[
X_t^\infty (\alpha, \beta) =  \sum_{i=1}^k \alpha_i X_{t-i} + \sum_{i=1}^q \beta_i \left( X_{t-i} - X_{t-i}^\infty (\alpha, \beta) \right) ,
\]
with initial condition $X_1^\infty (\alpha, \beta) = X_1$. We then denote by
\begin{equation}  \label{finft}
f^\infty_t (\alpha, \beta) = \ell_t \left( X_t ,  X_t^\infty (\alpha, \beta) \right)
\end{equation}
the loss  suffered by the prediction $X_t^\infty (\alpha, \beta)$ at iteration $t$.
From this definition it follows that $ X_t^\infty (\alpha, \beta)  $ is of the form
$
X_t^\infty (\alpha, \beta)  =\sum_{i=1}^{t-1} c_i(\alpha,\beta) X_{t-i}
$
for some appropriate coefficients $c_i(\alpha,\beta)$. The motivation behind the definition of $f^\infty_t$ follows from the need to replace $f_t$ with a loss function that fits the full information online optimization model (no unknown parameters).
We set $m \in \mathbb{N}$, and define
\[
X_t^m (\alpha, \beta)  =   \sum_{i=1}^k \alpha_i X_{t-i} + \sum_{i=1}^q \beta_i \left( X_{t-i} - X_{t-i}^{m-i} (\alpha, \beta) \right) ,
\]
with initial condition $X^m_t (\alpha, \beta) = X_t$ for all $t$ and $m \leq 0$. We denote by
\begin{equation} \label{fmt}
f^m_t (\alpha, \beta) = \ell_t \left( X_t ,  X_t^m (\alpha, \beta) \right)
\end{equation}
the loss  suffered by the prediction $X_t^m (\alpha, \beta)$ at iteration $t$.
The motivation here is simple: it is easier to generate predictions using only the last $\left( m+k \right) $ signals, and the distance between the loss function is relatively small.
Now, let
\begin{equation}  \label{star}
 \left( \alpha^\star, \beta^\star \right) = \arg \min_{\alpha,\beta} \sum_{t=1}^T  \mathbb{E} \left[ f_t(\alpha, \beta) \right]
\end{equation}
denote the best ARMA coefficient in hindsight for predicting the signal $\{ X_t \}_{t=1}^T$.

\noindent Then, from Lemma \ref{first}, stated and proven below, we have that
\[
\min_\gamma \sum_{t=1}^T \ell^m_{t}(\gamma) \leq \sum_{t=1}^T f^m_t \left( \alpha^\star, \beta^\star \right),
\]
and it follows that
\[
\sum_{t=1}^T \ell^m_{t}(\gamma^{t}) - \sum_{t=1}^T f^m_t \left( \alpha^\star, \beta^\star \right) = O \Big(\Big(GD+\frac{1}{\lambda}\Big) \log(T) \Big).
\]
From Lemma \ref{second} below we know that
\[
\left| \sum_{t=1}^T \mathbb{E} \left[ f^\infty_t\left( \alpha^\star, \beta^\star \right) \right] - \sum_{t=1}^T \mathbb{E} \left[ f^m_t\left( \alpha^\star, \beta^\star \right) \right] \right| = O(1) ,
\]
for $ m = q \cdot \log_{1-\varepsilon } \left( \left(  T L M_{\max}  \right)^{-1}\right) $, which implies that
\[
\sum_{t=1}^T \ell^m_{t}(\gamma^{t}) - \sum_{t=1}^T \mathbb{E} \left[ f^\infty_t \left( \alpha^\star, \beta^\star \right) \right] = O \Big(\Big(GD+\frac{1}{\lambda}\Big) \log(T) \Big).
\]
Finally, from Lemma \ref{third} below we know that
\[
\left| \sum_{t=1}^T \mathbb{E} \left[ f^\infty_t \left( \alpha^\star, \beta^\star \right) \right] - \sum_{t=1}^T \mathbb{E} \left[ f_t \left( \alpha^\star, \beta^\star \right) \right] \right|= O \big(1) ,
\]
and thus
\[
 \sum_{t=1}^T \ell^m_{t}(\gamma^{t}) - \min_{\alpha,\beta} \sum_{t=1}^T \mathbb{E} \left[ f_t(\alpha, \beta) \right]  = O \left( \left( GD+ \frac{1}{\lambda} \right) \log(T) \right).
\]
\end{proof}
Next, we prove the lemmas we used.
\begin{lemma} \label{first}
Let $\ell_t^m (\gamma)$, $f_t^m (\alpha, \beta) $ and $ \left( \alpha^\star, \beta^\star \right)$ be as denoted in Equations \ref{lmt}, \ref{fmt} and \ref{star}.
Then, for all $ m \in \mathbb{N} $ and data sequence $ \{X_t\}_{t=1}^{T}$ that satisfies the assumptions from Section \ref{sec:assumptions}, it holds that
\[
\min_\gamma \sum_{t=1}^T \ell^m_{t}(\gamma) \leq \sum_{t=1}^T f^m_t \left( \alpha^\star, \beta^\star \right).
\]
\end{lemma}

\begin{proof}
Note that if we set $\gamma_i^\star = c_i(\alpha^\star,\beta^\star)$, we immediately get that
\[
\sum_{t=1}^T \ell^m_{t}(\gamma^\star) = \sum_{t=1}^T f^m_t \left( \alpha^\star, \beta^\star \right).
\]
Trivially, it always holds that
\[
\min_\gamma \sum_{t=1}^T \ell^m_{t}(\gamma) \leq \sum_{t=1}^T \ell^m_{t}(\gamma^\star),
\]
which completes the proof.
\end{proof}
\begin{lemma} \label{second}
Let $f_t^\infty (\alpha, \beta) $, $f_t^m (\alpha, \beta) $ and $ \left( \alpha^\star, \beta^\star \right)$ be as denoted in Equations  \ref{finft}, \ref{fmt} and \ref{star}. Then, for any data sequence $ \{X_t\}_{t=1}^{T}$ that satisfies the assumptions from Section \ref{sec:assumptions}, it holds that
\[
\left| \sum_{t=1}^T \mathbb{E} \left[ f^\infty_t\left( \alpha^\star, \beta^\star \right) \right] - \sum_{t=1}^T \mathbb{E} \left[ f^m_t\left( \alpha^\star, \beta^\star \right) \right] \right| = O(1) ,
\]
if we choose  $ m = q \cdot \log_{1-\varepsilon } \left( \left(  T L M_{\max}  \right)^{-1}\right) $.
\end{lemma}
\begin{proof}
We set $t$, and look at the distance between $f^\infty_t(\alpha^\star, \beta^\star)$ and $f^m_t(\alpha^\star, \beta^\star)$ in expectation. We show by induction that 
\[
 \mathbb{E} \left[  | X_t^m \left( \alpha^\star, \beta^\star \right)  - X_t^\infty \left( \alpha^\star, \beta^\star \right) | \right] \leq  2M_{\max} \cdot \left( 1- \varepsilon \right) ^ {\frac{m}{q}} .
 \]
For $m=0$ we have that $ X_t^0 \left( \alpha^\star, \beta^\star \right) = X_t $ from the definition, and hence
\[
 | X_t^0 \left( \alpha^\star, \beta^\star \right)  - X_t^\infty \left( \alpha^\star, \beta^\star \right) | \leq  | X_t  - X_t^\infty \left( \alpha^\star, \beta^\star \right) | \leq | X_t  - X_t^\infty \left( \alpha^\star, \beta^\star \right) -\epsilon_t | + | \epsilon_t |  .
 \]
Now,  $ \mathbb{E} \left[ \left| \epsilon_t \right| \right] < M_{\max}< \infty $ for all $t$ and $\mathbb{E} \left[ | X_t  - X_t^\infty \left( \alpha^\star, \beta^\star \right) -\epsilon_t | \right] $ decays exponentially as proven in lemma \ref{third}, and hence the inductive basis holds for $m=0$. Next, we prove that the inductive basis holds for $m=1,\ldots,q-1 $:
\begin{eqnarray*}
\lefteqn{  | X^m_t  \left( \alpha^\star, \beta^\star \right)  - X_t^\infty  \left( \alpha^\star, \beta^\star \right)  |} \nonumber\\
& = & \left|  \sum_{i=1}^q \beta^\star_i \left( X_{t-i} -  X_{t-i}^{m-i} \left( \alpha^\star, \beta^\star \right) \right) - \sum_{i=1}^q \beta^\star_i \left( X_{t-i} -  X_{t-i}^\infty \left( \alpha^\star, \beta^\star \right) \right) \right|\nonumber\\
& = &  \left|  \sum_{i=1}^m \beta^\star_i \left(X_{t-i}^\infty \left( \alpha^\star, \beta^\star \right) -X_{t-i}^{m-i} \left( \alpha^\star, \beta^\star \right)\right) +  \sum_{i=m+1}^q \beta^\star_i \left(X_{t-i}^\infty \left( \alpha^\star, \beta^\star \right) -X_{t-i}^{m-i} \left( \alpha^\star, \beta^\star \right)\right) \right| \nonumber\\
& \stackrel{(1)}{\leq} &   \sum_{i=1}^m \left| \beta^\star_i  \right| \cdot  \left| X_{t-i}^\infty \left( \alpha^\star, \beta^\star \right) -X_{t-i}^{m-i} \left( \alpha^\star, \beta^\star \right) \right|  +\sum_{i=m+1}^q \left| \beta^\star_i  \right| \cdot  \left| X_{t-i}^\infty \left( \alpha^\star, \beta^\star \right) -X_t \right|  \nonumber\\
& \stackrel{(2)}{\leq} &   \sum_{i=1}^m \left| \beta^\star_i  \right|   \cdot 2M_{\max} \cdot \left( 1- \varepsilon \right) ^ {\frac{m-i}{q}} + \sum_{i=m+1}^q \left| \beta^\star_i  \right|   \cdot 2M_{\max}   \stackrel{(3)}{\leq} \sum_{i=1}^q \left| \beta^\star_i  \right|   \cdot 2M_{\max} \cdot \left( 1- \varepsilon \right) ^ {\frac{m-q}{q}}  \nonumber\\
& \leq &   2M_{\max} \cdot \left( 1- \varepsilon \right) ^ {\frac{m}{q}} . \nonumber
\end{eqnarray*}
$(1)$ is true from the triangle inequality and from the definition of $X_t^m$ for $m \leq 0 $. $(2)$ is true from the inductive hypothesis on $m$. $(3)$ is true since $ 1 \leq \left( 1- \varepsilon \right) ^ {\frac{m-q}{q}}$ for $m=1,\ldots,q-1 $.

\noindent For the inductive step we assume that
\[
| X_\tau^{\mu} \left( \alpha^\star, \beta^\star \right)  - X_\tau^\infty \left( \alpha^\star, \beta^\star \right) | \leq 2M_{\max} \cdot \left( 1- \varepsilon \right) ^ {\frac{\mu}{q}}
\]
 for $q \leq \mu < m $ and $\tau < t$, and prove that 
 \[
  | X^m_t  \left( \alpha^\star, \beta^\star \right)  - X_t^\infty  \left( \alpha^\star, \beta^\star \right)  | \leq 2M_{\max} \cdot \left( 1- \varepsilon \right) ^ {\frac{m}{q}} .
  \]
Thus, 
\begin{eqnarray*}
\lefteqn{  | X^m_t  \left( \alpha^\star, \beta^\star \right)  - X_t^\infty  \left( \alpha^\star, \beta^\star \right)  |} \nonumber\\
& = & \left|  \sum_{i=1}^q \beta^\star_i \left( X_{t-i} -  X_{t-i}^{m-i} \left( \alpha^\star, \beta^\star \right) \right) - \sum_{i=1}^q \beta^\star_i \left( X_{t-i} -  X_{t-i}^\infty \left( \alpha^\star, \beta^\star \right) \right) \right|\nonumber\\
& = &  \left|  \sum_{i=1}^q \beta^\star_i \left(X_{t-i}^\infty \left( \alpha^\star, \beta^\star \right) -X_{t-i}^{m-i} \left( \alpha^\star, \beta^\star \right)\right) \right|
\leq    \sum_{i=1}^q \left| \beta^\star_i  \right| \cdot  \left| X_{t-i}^\infty \left( \alpha^\star, \beta^\star \right) -X_{t-i}^{m-i} \left( \alpha^\star, \beta^\star \right) \right| \nonumber\\
& \leq &   \sum_{i=1}^q \left| \beta^\star_i  \right|   \cdot 2M_{\max} \cdot \left( 1- \varepsilon \right) ^ {\frac{m-i}{q}}  \leq \sum_{i=1}^q \left| \beta^\star_i  \right|   \cdot 2M_{\max} \cdot \left( 1- \varepsilon \right) ^ {\frac{m-q}{q}} \nonumber\\
& \leq & \left( 1- \varepsilon \right) \cdot 2M_{\max} \cdot \left( 1- \varepsilon \right) ^ {\frac{m-q}{q}}  =  2M_{\max} \cdot \left( 1- \varepsilon \right) ^ {\frac{m}{q}} , \nonumber
\end{eqnarray*}
which completes the induction.
Recall that $ \ell_t $ is Lipshitz continuous for some Lipshitz constant $ L > 0 $ from Assumption 2, and hence it follows that
\begin{eqnarray*}
\lefteqn{ \left| \mathbb{E} \left[ f^\infty_t\left( \alpha^\star, \beta^\star \right) \right]  - \mathbb{E} \left[ f^m_t\left( \alpha^\star, \beta^\star \right) \right]  \right| = \left| \mathbb{E} \left[ \ell_t \left( X_t , X_t^\infty \left( \alpha^\star, \beta^\star \right) \right) \right] - \mathbb{E} \left[ \ell_t \left( X_t , X_t^m  \left( \alpha^\star, \beta^\star \right) \right) \right]  \right|} \nonumber\\
& \leq &    \mathbb{E} \left[  \left| \ell_t \left( X_t , X_t^\infty  \left( \alpha^\star, \beta^\star \right) \right) - \ell_t \left( X_t , X_t^m  \left( \alpha^\star, \beta^\star \right) \right) \right|  \right]  \leq L \cdot  \mathbb{E} \left[ | X^m_t  \left( \alpha^\star, \beta^\star \right)  - X_t^\infty  \left( \alpha^\star, \beta^\star \right) |  \right] \nonumber\\
&  \leq &    L \cdot 2M_{\max} \cdot \left( 1- \varepsilon \right) ^ {\frac{m}{q}} , \nonumber
\end{eqnarray*}
where the first inequality follows from Jensen's inequality. By summing the above for all $t$ we get that
\[
\left| \sum_{t=1}^T \mathbb{E} \left[ f^\infty_t\left( \alpha^\star, \beta^\star \right) \right] - \sum_{t=1}^T \mathbb{E} \left[ f^m_t\left( \alpha^\star, \beta^\star \right) \right] \right| \leq TL \cdot 2M_{\max} \cdot \left( 1- \varepsilon \right) ^ {\frac{m}{q}}.
\]
Finally, choosing  $ m = q \cdot \log_{1-\varepsilon } \left( \left(  T L M_{\max}  \right)^{-1}\right)  $ yields
\[
\left| \sum_{t=1}^T \mathbb{E} \left[ f^\infty_t\left( \alpha^\star, \beta^\star \right) \right] - \sum_{t=1}^T \mathbb{E} \left[ f^m_t\left( \alpha^\star, \beta^\star \right) \right] \right| = O(1).
\]
 \end{proof}
\begin{lemma} \label{third}
Let $f_t (\alpha, \beta) $, $f_t^\infty (\alpha, \beta) $ and $ \left( \alpha^\star, \beta^\star \right)$ be as denoted in Equations \ref{ft}, \ref{finft} and \ref{star}.
Then, for any data sequence $ \{X_t\}_{t=1}^{T}$ that satisfies the assumptions from Subsection \ref{sec:assumptions}, it holds that
\[
\left| \sum_{t=1}^T \mathbb{E} \left[ f^\infty_t \left( \alpha^\star, \beta^\star \right) \right] - \sum_{t=1}^T \mathbb{E} \left[ f_t \left( \alpha^\star, \beta^\star \right) \right] \right|= O \left(1 \right).
\]
\end{lemma}
\begin{proof}
First, denote by $\left( \alpha',\beta' \right)$ the coefficient vectors that have generated the signal. Trivially, it holds that
\[
\sum_{t=1}^T f_t \left( \alpha',\beta' \right) = \sum_{t=1}^T \ell_t \left( X_t , X_t - \epsilon_t \right) .
\]
From Assumption 1, $\epsilon_t$ is independent of $\epsilon_1,\ldots,\epsilon_{t-1}$, and hence the best prediction available at time $t$ will cause a loss of at least $ \ell_t \left( X_t , X_t - \epsilon_t \right) $ in expectation. We can think of it in the following way: at time $t$, the online player has no previous information regarding the adversary's choice of $\epsilon_t$. Since $ \mathbb{E} \left[ \epsilon_t \right] = 0 $ and $ \ell_t $ is convex, predicting the expected signal is the optimal policy of the online player at time $t$. It follows that $ \left( \alpha^\star, \beta^\star \right) = \left( \alpha',\beta' \right) $, meaning the best ARMA coefficients in hindsight are those that have generated the signal.

Next, we show by induction that $ \mathbb{E} \left[ | X_t - X_t^\infty \left( \alpha^\star, \beta^\star \right)  - \epsilon_t | \right]$ decays exponentially as $t$ grows linearly. Without loss of generality, we can assume that for $t = 1, \dots, q $ we have that $ \mathbb{E} \left[ | X_t -X_t^\infty \left( \alpha^\star, \beta^\star \right)  - \epsilon_t | \right] <  \rho $ for some $ \rho >0 $, as the inductive basis.
Now, for the inductive step we assume that
\[
\mathbb{E} \left[  \left| X_\tau - X_\tau^\infty \left( \alpha^\star, \beta^\star \right)  - \epsilon_\tau \right| \right] <   \rho \cdot \left(1- \varepsilon \right) ^{\frac{\tau}{q}}
 \]
 for $q<\tau < t$, and prove that 
 \[
  \mathbb{E} \left[ | X_t - X_t^\infty  \left( \alpha^\star, \beta^\star \right)  - \epsilon_t | \right] \leq \rho \cdot \left(1- \varepsilon \right) ^{\frac{t}{q}} .    
  \]
Thus,
\begin{eqnarray*}
\lefteqn{ \mathbb{E} \left[ | X_t -X_t^\infty  \left( \alpha^\star, \beta^\star \right)  - \epsilon_t | \right]} \nonumber\\
& = & \mathbb{E} \left[ \left| \sum_{i=1}^k \alpha^\star_i X_{t-i} +  \sum_{i=1}^q \beta^\star_i \epsilon_{t-i} + \epsilon_t - \sum_{i=1}^k \alpha^\star_i X_{t-i} - \sum_{i=1}^q \beta^\star_i \left( X_{t-i} -  X_{t-i}^\infty \left( \alpha^\star, \beta^\star \right) \right)  - \epsilon_t  \right|  \right] \nonumber\\
& = &  \mathbb{E} \left[ \left|  \sum_{i=1}^q \beta^\star_i \left(X_{t-i}^\infty \left( \alpha^\star, \beta^\star \right) -X_{t-i} - \epsilon_{t-i} \right) \right| \right]
\leq    \sum_{i=1}^q \left| \beta^\star_i  \right| \cdot \mathbb{E} \left[  \left| X_{t-i}^\infty \left( \alpha^\star, \beta^\star \right) -X_{t-i} - \epsilon_{t-i} \right| \right] \nonumber\\
& \leq &   \sum_{i=1}^q \left| \beta^\star_i  \right|   \cdot \rho \cdot \left(1- \varepsilon \right) ^{\frac{t-i}{q}}   \leq \sum_{i=1}^q \left| \beta^\star_i  \right|   \cdot  \rho \cdot \left(1- \varepsilon \right) ^{\frac{t-q}{q}}  \leq  \left( 1 - \varepsilon \right)    \cdot   \rho \left(1- \varepsilon \right) ^{\frac{t-q}{q}}  =   \rho \cdot\left(1- \varepsilon \right) ^{\frac{t}{q}} \nonumber
\end{eqnarray*}
which ends the induction.
Recall that $ \ell_t $ is assumed to be Lipshitz continuous for some constant $ L > 0 $, and hence it follows that
\begin{eqnarray*}
\lefteqn{ \left| \mathbb{E} \left[ f^\infty_t \left( \alpha^\star, \beta^\star \right) \right] - \mathbb{E} \left[ f_t\left( \alpha^\star, \beta^\star \right) \right]  \right| = \left| \mathbb{E} \left[ \ell_t \left( X_t ,  X_t^\infty \left( \alpha^\star, \beta^\star \right) \right) \right]- \mathbb{E} \left[ \ell_t \left( X_t , X_t - \epsilon_t \right) \right]  \right|} \nonumber\\
& = & \left| \mathbb{E} \left[  \ell_t \left( X_t ,  X_t^\infty \left( \alpha^\star, \beta^\star \right) \right) - \ell_t \left( X_t , X_t - \epsilon_t \right) \right]  \right|    \nonumber\\
& \leq &    \mathbb{E} \left[  \left| \ell_t \left( X_t ,  X_t^\infty \left( \alpha^\star, \beta^\star \right) \right) - \ell_t \left( X_t , X_t - \epsilon_t \right) \right|  \right]  \leq L \cdot \mathbb{E} \left[  | X_t -X_t^\infty \left( \alpha^\star, \beta^\star \right)  - \epsilon_t |  \right]  \nonumber\\
& \leq &   \rho L \cdot\left(1- \varepsilon \right) ^{\frac{t}{q}} . \nonumber
\end{eqnarray*}
Finally, summing over all iterations yields
\[
\left| \sum_{t=1}^T \mathbb{E} \left[ f^\infty_t \left( \alpha^\star, \beta^\star \right) \right] - \sum_{t=1}^T \mathbb{E} \left[ f_t \left( \alpha^\star, \beta^\star \right) \right] \right| =  O \left( 1 \right) .
\]
\end{proof}
\textbf{Remark}: In Lemma \ref{third} we assume here that $  \rho q L  = O \left( 1 \right) $. Otherwise, an element of $O \left( \rho q L \right) $ is added to the regret bound in Theorems \ref{main}, which does not affect the asymptotic result.

\subsection{ARMA Online Gradient Descent (ARMA-OGD)} \label{sec:OGD}
We now turn to present a different algorithm for choosing $\gamma^t$ at each time point. This algorithm is applicable to general convex loss functions, as well as to exp-concave ones. It is computationally simpler but has a somewhat worse theoretical (and empirical) performance compared to the previous one, when considering an exp-concave loss function. The notation $\Pi_\mathcal{K}$ refers to the Euclidean projection onto $\mathcal{K}$, i.e., $ \Pi_\mathcal{K} (y) = \arg \min_{x \in \mathcal{K}} \| y-x \| _2 $ . 

\begin{algorithm}[h!]
\caption{ARMA-OGD(k,q)}
\label{alg:ogd}
\begin{algorithmic}[1]
\STATE Input: ARMA order k,q. Learning rate $\eta$.
\STATE Set $ m = q \cdot \log_{1-\varepsilon } \left( \left(  T L M_{\max}  \right)^{-1}\right)  $.
\STATE Choose $\gamma^1 \in \mathcal{K}$ arbitrarily.
\FOR {$t=1$ to $(T-1)$}
\STATE Predict $ \tilde{X}_t(\gamma^{t}) = \sum_{i=1}^{m+k} \gamma^{t}_i X_{t-i}$.
\STATE Observe $X_t$ and suffer loss $\ell^m_t(\gamma_t)$.
\STATE Let $\nabla_t = \nabla \ell^m_t(\gamma^t)$
\STATE Set  $\gamma^{t+1} \leftarrow \Pi_\mathcal{K} \Big( \gamma^{t}-\frac{1}{\eta}  \nabla_t \Big)$
\ENDFOR
\end{algorithmic}
\end{algorithm}

\noindent \\For Algorithm \ref{alg:ogd} we can prove the following:
\begin{theorem} \label{main2}
Let $k,q \geq 1$, and set $ \eta = \frac{D}{G \sqrt{T}} $. Then, for any data sequence $ \{X_t\}_{t=1}^{T}$ that satisfies the assumptions from Section \ref{sec:assumptions}, Algorithm \ref{alg:ogd} generates an online sequence $\{\gamma^t\}_{t=1}^T$, for which the following holds:
\begin{equation} \label{eq:ogd}
 \sum_{t=1}^T \ell^m_{t}(\gamma^{t}) - \min_{\alpha,\beta} \sum_{t=1}^T \mathbb{E} \left[ f_t(\alpha, \beta) \right]  = O \left( GD \sqrt{T} \right) .
\end{equation}
\end{theorem}
The proof of this theorem is very similar to the proof of Theorem \ref{main}, albeit plugging into our framework the Online Gradient Descent (OGD) algorithm of \cite{Zinkevich03} rather than the Online Newton Step algorithm.

\section{Additional results}\label{sec:add}
In this section we present an analysis for the case when the noise terms are allowed to be adversarial, and also an application of Theorem \ref{main} for squared loss.

\subsection{Adversarial noise}\label{sec:advnoise}
The results presented in Theorems \ref{main} and \ref{main2} rely on the assumptions that the noise terms are independent and zero-mean. Under these assumptions, the best coefficient vectors in hindsight are those that have generated the signal. However, if we allow the noise terms to be adversarially generated (the adversary chooses $\epsilon_t$ at time $t$ with no limitations), the best coefficient vectors in hindsight are not necessarily the ones used for generating the  signal. For this case we have the following theorem:
\begin{theorem} \label{add}
Denote by $\left( \alpha',\beta' \right)$ the coefficient vectors that have generated the signal, and assume that $ \{X_t\}_{t=1}^{T} $ satisfies Assumptions 2-5 from Section \ref{sec:assumptions}, when the noise terms are allowed to be chosen adversarially. Then, for exp-concave loss functions Algorithm \ref{alg1} generates an online sequence $\{\gamma^t\}_{t=1}^T$, for which the following holds:
\begin{eqnarray*}
 \sum_{t=1}^T \ell^m_{t}(\gamma^{t}) - \sum_{t=1}^T \mathbb{E} \left[ f_t \left( \alpha',\beta' \right) \right]  = O \left( \left( GD+ \frac{1}{\lambda} \right) \log(T) \right) ,
\end{eqnarray*}
and for convex loss functions, Algorithm \ref{alg:ogd} generates an online sequence $\{\gamma^t\}_{t=1}^T$, for which the following holds:
\begin{eqnarray*}
 \sum_{t=1}^T \ell^m_{t}(\gamma^{t}) - \sum_{t=1}^T \mathbb{E} \left[ f_t \left( \alpha',\beta' \right) \right]  = O \left( GD \sqrt{T} \right) .
\end{eqnarray*}
\end{theorem}
Notice that we compare here the total loss suffered by our algorithms to the expected loss suffered by ARMA prediction with the coefficient vectors that have generated the signal, and not to the expected loss of the best ARMA prediction in hindsight. Nevertheless, this theorem captures interesting cases (e.g., correlated noise), in which traditional approaches fail to perform properly.
The proof of this theorem resembles  the proof of Theorem \ref{main}, with the modification of plugging $  \left( \alpha',\beta' \right) $ into Lemmas \ref{second} and  \ref{third}, instead of $ \left( \alpha^\star, \beta^\star \right) $.

\subsection{Application of Theorem \ref{main} to squared loss }\label{sec:sqr}
As already mentioned, the squared loss is the most commonly used loss function in time series analysis. It is defined as $\ell_t ( X_t , \tilde{X}_t ) =  ( X_t - \tilde{X}_t ) ^2 $ 
for prediction $ \tilde{X}_t$ and signal $X_t$. In our case, the predictions come from an AR model with horizon $\left( m + k \right)$, and hence our loss at time $t$ is $ ( X_t - \sum_{i=1}^{m+k} \gamma_i^t X_{t-i} )^2$, when $\{\gamma^t\}_{t=1}^T$ are generated using Algorithm \ref{alg1}. Substituting the values of $G$, $D$ and $\lambda$, as defined and computed in Subsection \ref{sec:params} for the squared loss, yields the following result:
\begin{equation} \label{eq:sq_loss}
 \sum_{t=1}^T \ell^m_{t}(\gamma^{t}) - \min_{\alpha,\beta} \sum_{t=1}^T \mathbb{E} \left[ f_t(\alpha, \beta) \right]  = O \left( k \log \left(T\right) + q \log ^2 \left(T\right) \right) .
\end{equation}
This result implies that the average loss suffered by Algorithm \ref{alg1} converges asymptotically to the average loss suffered by the best ARMA prediction in hindsight, under the assumptions from Section \ref{sec:assumptions}. In section \ref{sec:experiments} we empirically verify this theoretical result, under some different settings.

%\begin{lemma} \label{fourth}
%Let $f_t (\alpha, \beta) $ and $f_t^\infty (\alpha, \beta) $ be as denoted in Equations \ref{ft} and \ref{finft}, and  $\left( \alpha',\beta' \right)$ be the coefficient vectors that have generated the signal.
%Then, for any data sequence $ \{X_t\}_{t=1}^{T}$ that satisfies assumptions 2-5 from Subsection \ref{sec:assumptions} (the noise terms are allowed to be adversarial), it holds that
%\[
%\sum_{t=1}^T f^\infty_t \left( \alpha',\beta' \right) -\sum_{t=1}^T \mathbb{E} \left[ f_t \left( \alpha',\beta' \right) \right]  = O \big(1).
%\]
%\end{lemma}
%The proof is provided in the supplementary material.

\section{Experiments}\label{sec:experiments}

The following experiments demonstrate the prediction effectiveness of the proposed algorithms, under some different settings. We compare the performance to the ARMA-RLS algorithm, which was presented in \cite{DinShiChe06}. In a few words, the ARMA-RLS is a ``proper learning" algorithm --- it tries to mimic the underlying model. It estimates the noise terms using a recursive least squares based method, and satisfies a prediction using these estimations and the previous signals. The ARMA-RLS does not assume noise stationarity or ergodicity. We also benchmark the standard Yule-Walker estimation method\footnote{Yule-Walker estimation method is offline. We use it as an online prediction method by a simple adaptation --- we let it predict the signal at time $t$ with the knowledge of the signal at times $1,\ldots,t-1$.}.
The results are displayed in the figures below. In all cases, the $x$-axis is time (number of samples), and the $y$-axis is the average squared loss.

\subsection{Experiments with artificial data}\label{subsec:experiments1}
In all experimental settings below we have averaged the results over 20 runs for stability. Also, we choose the order of our AR prediction to be $ m+k = 10$ in all settings.   \\
%OS: If available and we still have time, I recommend plotting standard deviation intervals
%SM: Oren: can you do it?
\begin{figure}[ht!]
     \begin{center}
        \subfigure[\textbf{Setting 1.} Sanity check]{%
            \label{fig:1}
            \includegraphics[width=0.5\textwidth]{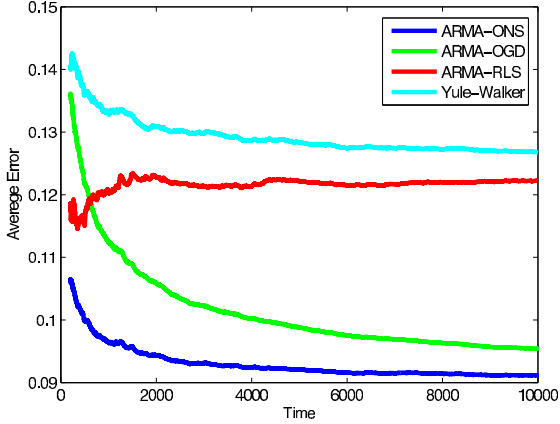}
        }%
        \subfigure[\textbf{Setting 2.} Slowly changing coefficients]{%
           \label{fig:2}
           \includegraphics[width=0.5\textwidth]{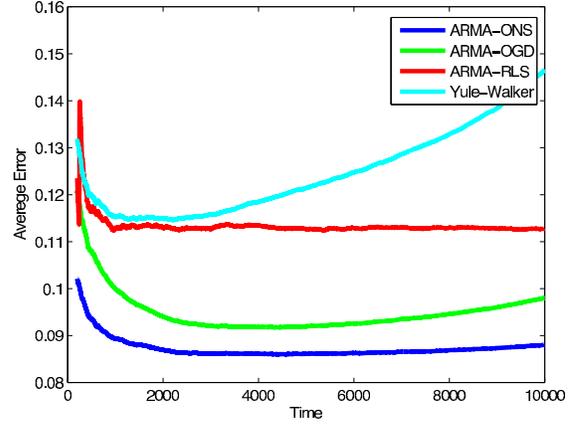}
        }\\ %  ------- End of the first row ----------------------%
        \subfigure[\textbf{Setting 3.}  Abrupt change ]{%
            \label{fig:3}
            \includegraphics[width=0.5\textwidth]{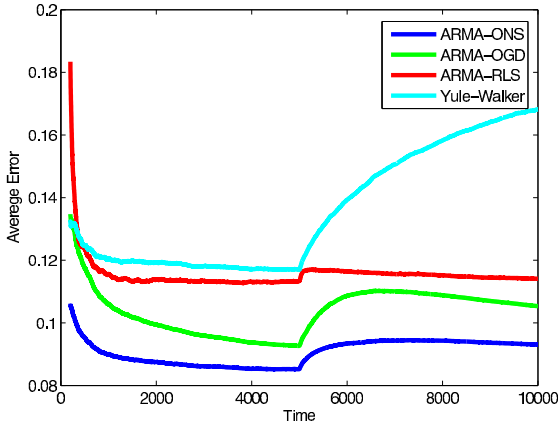}
        }%
        \subfigure[\textbf{Setting 4.}  Correlated noise ]{%
            \label{fig:4}
            \includegraphics[width=0.5\textwidth]{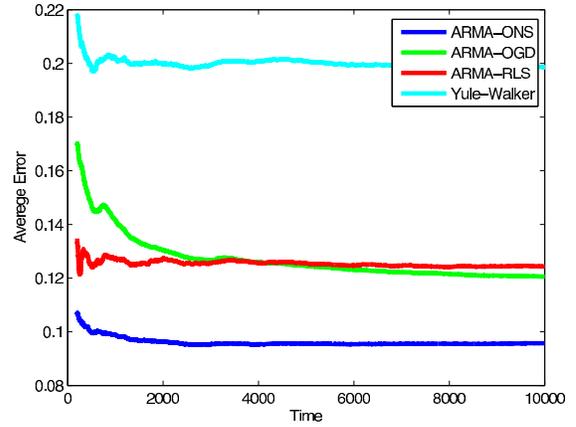}
        }%
    \end{center}
    \caption{%
        Experimental results for artificial data, all averaged over 20 runs.
     }%
   \label{fig:subfigures}
\end{figure}

\noindent \textbf{Setting 1.}
We started with a simple sanity check using Gaussian noise. We generated a stationary ARMA process using the coefficient vectors $\alpha = [0.6,-0.5,0.4,-0.4,0.3]$ and $\beta = [0.3,-0.2]$, when the noise terms are uncorrelated and normally distributed as $\mathcal{N}(0,0.3^2)$.
Note that since predicting the noise is impossible, a perfect predictor will suffer an average error rate of at least the variance of the noise --- 0.09 in this setting. As can be seen in Figure \ref{fig:1} the ARMA-ONS algorithm outperforms the other online algorithms due to its lower regret in this setting of exp-concave loss functions, and quickly approaches the performance of the perfect predictor.\\

\noindent \textbf{Setting 2.}
We generated the non-stationary ARMA process using the coefficient vectors $\beta = [0.32,-0.2]$ and
\begin{eqnarray*}
\alpha (t) & = & [-0.4, -0.5, 0.4, 0.4, 0.1] * \Big( \frac{t}{10^4} \Big) + [0.6, -0.4, 0.4, -0.5, 0.4] * \Big( 1 - \frac{t}{10^4} \Big),
\end{eqnarray*}
i.e., the coefficient vectors change slowly in time. The noise terms are uncorrelated and distributed uniformly on $[-0.5,0.5]$ (denoted as $Uni[-0.5,0.5]$).
In this setting, a perfect predictor will suffer average error rate of at least 0.0833, due to the variance of the noise. The motivation behind this setting is to demonstrate the effectiveness of the online algorithms in the non-stationary case, in which the coefficients change in time. This is especially important when dealing with real data time series, since the stationarity assumption is rather strict.
In Figure \ref{fig:2} we can see the clear advantage of our online algorithms. Here again, ARMA-ONS is superior to the other algorithms, despite it being less adaptive --- as  the theoretical bounds predict; see \cite{HaSe09} for discussion of adaptivity of OGD vs.~ONS.\\

\noindent \textbf{Setting 3.}
Here we consider the non-stationary ARMA process that is generated using two different sets of coefficient vectors. The first set is $\alpha = [0.6,-0.5,0.4,-0.4,0.3]$ and $\beta = [0.3,-0.2]$, and it is used for generating the signal at the first half of the iterations. The second set is $\alpha = [-0.4, -0.5, 0.4, 0.4, 0.1]$ and $\beta = [-0.3,0.2]$, and it is used for generating the signal at the second half of the iterations. The noise terms are uncorrelated and distributed $Uni[-0.5,0.5]$.
In Figure \ref{fig:3} we demonstrate the effectiveness of online algorithms in a scenario when the coefficients abruptly change. Here again, a perfect predictor will suffer average error rate of at least 0.0833, due to the variance of the noise.\\

\noindent \textbf{Setting 4.}
Consider an ARMA process that is generated using the coefficient vectors $\alpha = [0.11, -0.5]$ and $\beta = [0.41, -0.39, -0.685, 0.1]$. Each noise term is distributed normally, with expectation that is the value of the previous noise term, and variance $0.3^2$. I.e., the noise terms are positively correlated.
In Figure \ref{fig:4} one can clearly see the robustness of online algorithms to correlated noise. Note that despite the correlativity introduced in this setting, ARMA-ONS achieves an average error rate that converges approximately to the variance of the noise --- 0.09. 

\subsection{Experiments with real data}\label{subsec:experiments2}
In this section we provide some preliminary results on real data time series, and show that for such data as well, our online learning approach is reasonably effective compared to existing approaches. For robustness, we consider time series from different fields.
\begin{figure}[ht!]
     \begin{center}
        \subfigure[Monthly average temperature]{%
            \label{fig:5}
            \includegraphics[width=0.5\textwidth]{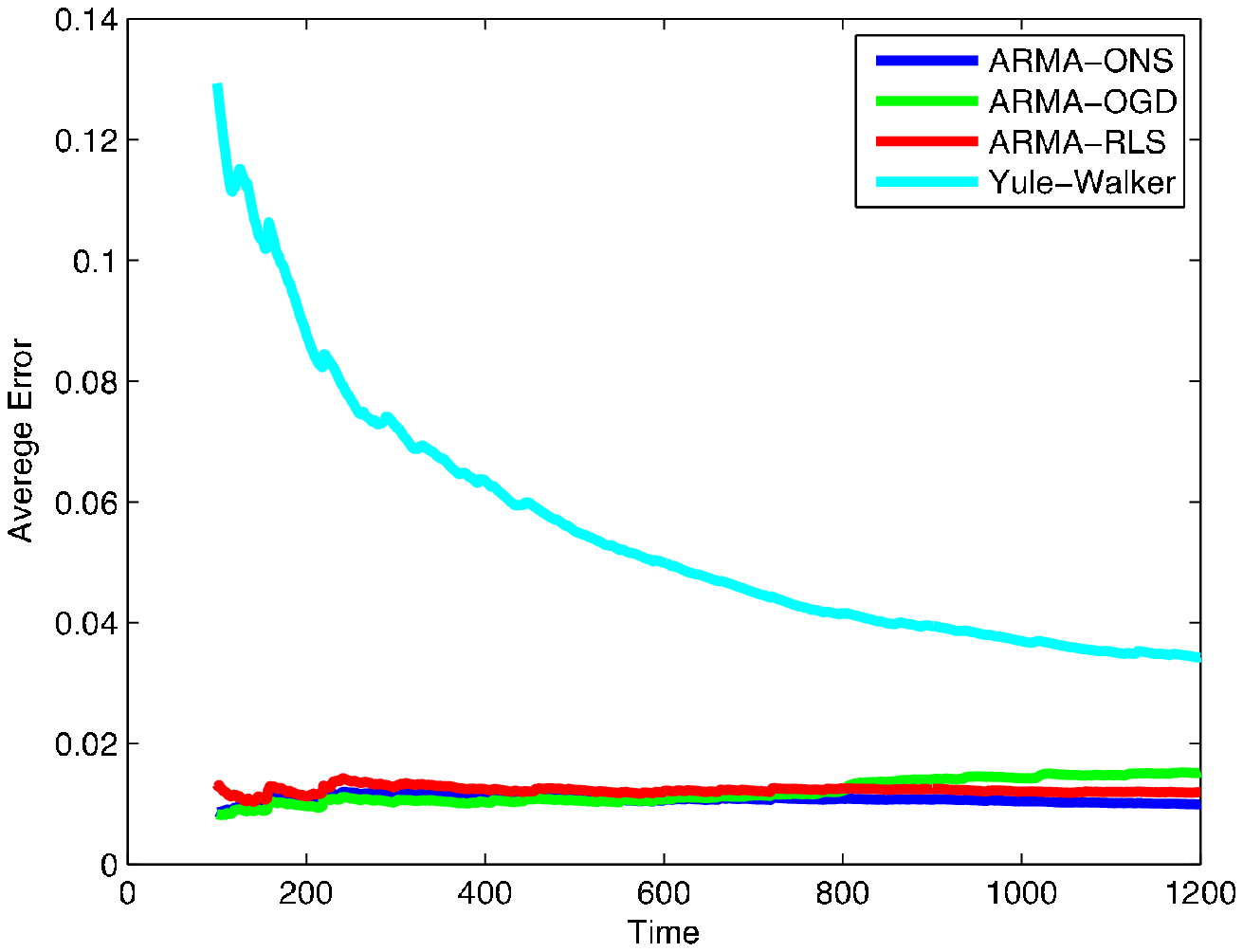}
        }%
        \subfigure[S\&P 500 daily returns]{%
           \label{fig:6}
           \includegraphics[width=0.5\textwidth]{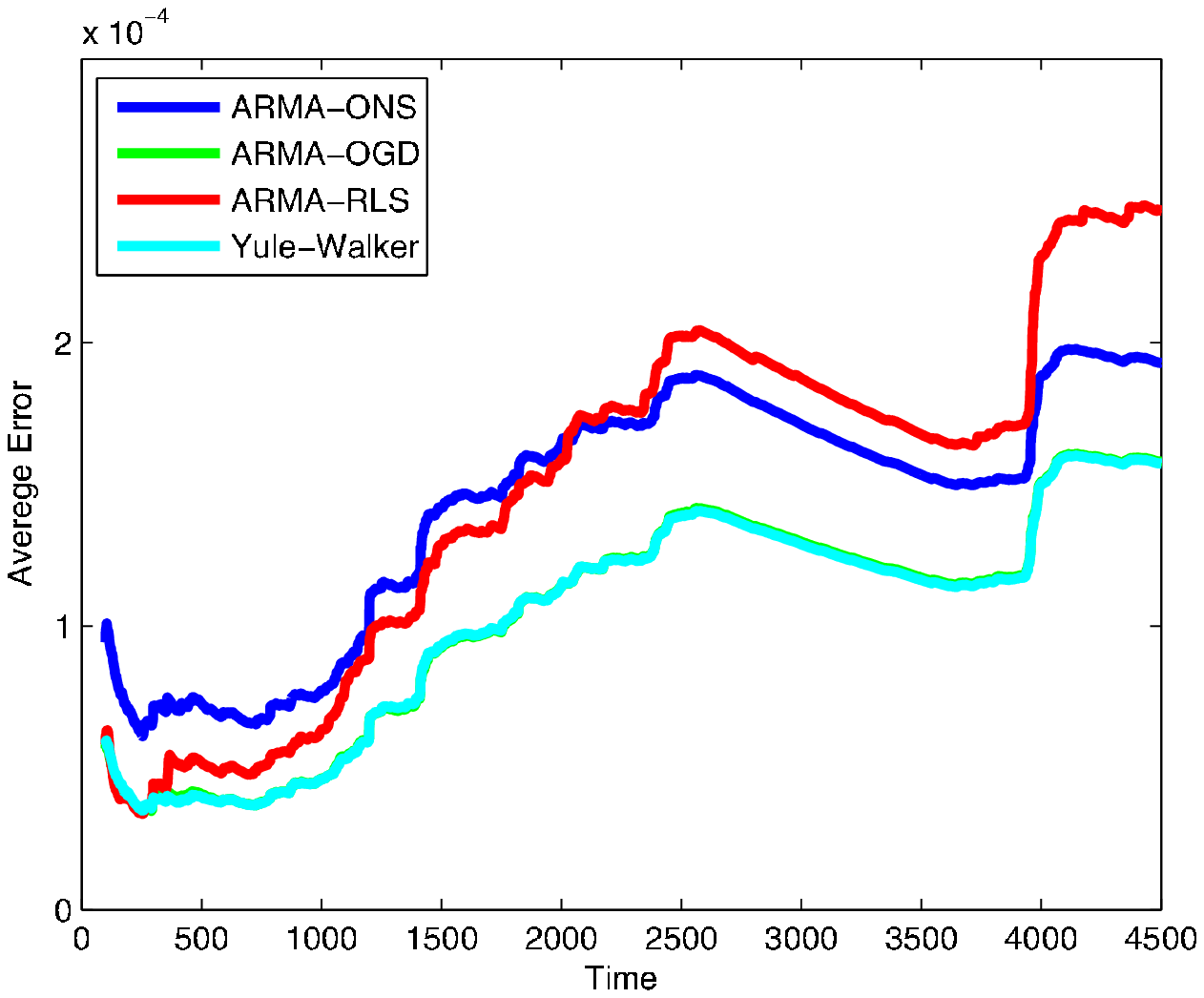}
        }\\ %  ------- End of the first row ----------------------%
    \end{center}
    \caption{%
        Experimental results for real data.
     }%
   \label{fig:subfigures2}
\end{figure}

The first time series is taken from the field of weather research. Each data point in this time series is the monthly average temperature of the sea surface, measured at a specific point. The data is taken from the \href{http://www.esrl.noaa.gov/psd/gcos_wgsp/Timeseries/GLBTS/}{Global Climate Observing System (GCOS)} website.
Since we are dealing with a weather related time series, and considering the monthly average temperature, it is rather reasonable that the time series follows a certain pattern. As can be seen in Figure \ref{fig:5}, this pattern can be well learned using the ARMA model by all four algorithms. However, the results clearly indicate the superiority of online algorithms.

% We note that while the three online algorithms perform roughly the same, ARMA-ONS is slightly better than the other two.
%SM: Explain more about the data
%OS: We need to cite where the data is taken from! Also, I agree with Shie that one should say what is this data - is this world average sea temperature? Temperatures in the Mexican Gulf? Something else?
%OA: How should the citation look like? link to the internet site?
%SM: Give a url - this is enough.
%OA: I removed the above sentence. I think that for a single run, with only one set of parameters it has no meaning.

%OS: Again, need to cite where the data is taken from - as Shie said, a URL is also OK.
The second time series is taken from the field of finance. Each data point in this time series is the daily return of the S\&P 500 index. The data is taken from \href{http://finance.yahoo.com/q/hp?s=SPY+Historical+Prices}{Yahoo! Finance}.
The results in Figure \ref{fig:6} indicate that the ARMA model is probably not a good model for predicting the returns of the S\&P 500 index. A possible reason is that the ARMA model is not rich enough, i.e., knowing the history of returns is not sufficient for satisfying a good prediction. The fact that offline familiar methods also fail here, strengthens this claim. See Section \ref{conclusion} for further discussion about fitting a time series model for financial data.

\section{Conclusion and discussion} \label{conclusion}
In this paper we developed a new approach for time series analysis --- an online learning approach. Our main result in this paper is that one can predict time series as well as the best ARMA model,  regardless of the loss function considered, under weak assumptions on the noise terms --- zero mean distribution. This result is strengthened in light of the fact that the noise terms in the underlying model are unknown to us at any stage. We overcome this difficulty by using improper learning techniques. Additionally, we present an analytical extension of our approach to adversarially generated noise terms. The main powerful properties of the online approach, as pointed out in our work, are generality, simplicity and efficiency, in comparison to existing methods.

There are three issues that remain for further research.
First, in our analysis we assume that $ \sum_{i=1}^q | \beta_i | < 1-\varepsilon$ for some $ \varepsilon > 0 $, which seems to limit the freedom of the $\beta$ coefficients. This assumption appears sometimes in the literature (e.g. in GARCH models) and is a sufficient condition for the MA component to be causally invertible, yet not necessary. In our case, we believe that this assumption follows from our proof techniques and the results would still hold for any $\beta$ coefficients. 
Second, in Section \ref{sec:add} we present results in which the total loss suffered by our algorithms is compared to the expected loss suffered by ARMA prediction with the coefficient vectors that have generated the signal.
Whereas competing against the best ARMA prediction under adversarial noise is impossible because of identifiability issues, it would be interesting to study intermediate setups such as correlated or adversarial noise to some extent.
Third, the ARMA model is not compatible for any time series, as can be seen in Section \ref{subsec:experiments2}, when a finance related time series is considered. However, \cite{Engle1982} showed that some finance related time series can be well predicted using the ARCH model and its expansions. Therefore, it would be interesting to generalize our work to other time series models, such as ARCH and ARIMA.

%
%% Acknowledgments---Will not appear in anonymized version
%\acks{This work was carried out and supported by the Technion-Microsoft Electronic Commerce Research Center.}

\newpage

\bibliography{arma_arxiv}
\bibliographystyle{alpha}

\end{document}